\documentclass{article}
\usepackage[utf8]{inputenc}
\usepackage[english]{babel}

\usepackage{blindtext}
\usepackage{bm}
\usepackage{float}

\usepackage{amssymb}
\usepackage[linesnumbered,ruled,vlined]{algorithm2e}
\usepackage{graphicx}

\usepackage{subcaption}
\usepackage{mathtools}

\DeclarePairedDelimiter\floor{\lfloor}{\rfloor}

\usepackage{amsthm}
\usepackage{caption}
\newtheorem{theorem}{Theorem}[section]
\newtheorem{corollary}{Corollary}[theorem]
\newtheorem{lemma}[theorem]{Lemma}

\theoremstyle{remark}
\newtheorem*{remark}{Remark}

\theoremstyle{definition}

\newtheorem*{definition*}{Definition}

\title{Learning Halfspaces With Membership Queries}
\author{Ori Kelner}
\date{October  2020}

\begin{document}

\maketitle

\section{ABSTRACT}
Active learning is a subfield of machine learning, in which the learning algorithm is allowed to choose the data from which it learns.
In some cases, it has been shown that active learning can yield an exponential gain in the number of samples the algorithm needs to see, in order to reach generalization error $\leq \epsilon$.
In this work we study the problem of learning halfspaces with membership queries. In the membership query scenario, we allow the learning algorithm to ask for the label of every sample in the input space. We suggest a new algorithm for this problem, and prove it achieves a near optimal label complexity in some cases. We also show that the algorithm works well in practice, and significantly out-performs uncertainty sampling.

\section{INTRODUCTION}
In recent years, active learning gained a lot of attention for its ability to reduce the amount of labeled data needed in order to learn a model. The key idea is that the learning algorithm is allowed to choose the data from which it learns. The three main settings that have been considered in the literature are: \textbf{pool-based sampling}, \textbf{stream-based selective sampling} and \textbf{membership query synthesis} \cite{Settles}.
In \textbf{pool-based sampling} we assume that there is a small set of labeled data $\mathcal{L}$ and a large pool of unlabeled data $\mathcal{U}$ available. Queries are selectively drawn from the pool. The key assumption behind \textbf{stream-based selective sampling}, is that obtaining an unlabeled instance is free (or inexpensive), so it can first be sampled from the actual distribution, and then the learner can decide whether or not to request its label. Lastly, as mentioned earlier, in \textbf{membership query synthesis}, the learner may request labels for any unlabeled instance in the input space, including (and typically assuming) queries that the learner generates de novo, rather than those sampled from some underlying natural distribution\cite{Settles}.
In this work, we will present an algorithm for learning homogeneous halfspaces, with membership queries.
Learning halfspaces in different active learning settings is very common in the literature \cite{10.1145/130385.130417, 10.1162/153244302760185243, 10.5555/1577069.1577080,10.5555/2886521.2886666}, and it was shown that it's possible to achieve exponential improvement over the usual sample complexity of supervised learning.   

Our main contributions are as follows:
\begin{enumerate}
    \item We will show that the algorithm we suggest has a near optimal label bound for the task of learning halfspaces under the uniform distribution. As far as we know, this is the first membership query algorithm for learning halfspaces with label bound guarantees. Furthermore, as we will show, the bound is not probabilistic, and is guaranteed to hold with every run of the algorithm, unlike previous results.
    \item In practice, our algorithm out-preforms uncertainty sampling significantly, achieving exponential decay in the number of samples needed in order to achieve generalization error $\leq$ $\epsilon$.
\end{enumerate}

\section{RELATED WORKS}

\subsection*{ACTIVE LEARNING OF HALFSPACES}
Active learning of homogeneous halfspaces under the uniform distribution is a known example where selective sampling can yield an exponential gain over supervised learning \cite{10.1145/130385.130417, 10.5555/1577069.1577080,10.1007/978-3-540-72927-3_5}. 
all \cite{10.1145/130385.130417, 10.5555/1577069.1577080,10.1007/978-3-540-72927-3_5} showed an exponential gain in the the stream-based scenario.

In \cite{10.1145/130385.130417} the authors used the Query-By-Committee (QBC) algorithm. The approach involves maintaining a committee $\mathcal{C} = \{f_1, \ldots , f_n\}$ of models which are all trained on the current labeled set $\mathcal{L}$, but represent competing hypotheses. Each committee member is then allowed to vote on the labelings of query candidates. The most informative query is considered to be the instance about which they most disagree. Conversely, \cite{10.5555/1577069.1577080} suggested a modification to the perceptron algorithm, namely: only querying samples that are close to the decision boundary, and also modified the update rule. 

Most similar in spirit to this work is \cite{10.5555/2886521.2886666}. \cite{10.5555/2886521.2886666} also used membership queries in order to learn halfspaces. They suggested to approximate the version space with an ellipsoid, and showed that in practice their method preforms well. 

\subsection*{SIMPLEX BISECTION METHOD}
The algorithm we suggest in this paper, is an adaptation of an old algorithm used for bisecting simplices. The algorithm was initially used for computing the roots of a continuous map, defined on a simplex \cite{10.2307/2006341}. In a nutshell, the algorithm starts with a simplex, and in every iteration bisects it to two parts, by cutting the longest edge of the simplex exactly in the middle. For a more detailed explanation see \cite{10.2307/2006341}.

\section{PRELIMINARIES AND NOTATION}

We are going to denote vectors in $\mathbb{R}^{n}$ with bold font, $\bm{x}$, and the corresponding coordinates with $x_i$.
We will denote by $S_n$  an n-simplex in $\mathbb{R}^{n}$.
Throughout this paper we will use $||\cdot||$ to denote the L2-norm. Denote the diameter of the simplex by  $d(S_n) = diam(S_n) =max\{||x-y|| \mid{} x,y \in S_n\}$ and observe that the diameter is equal to the longest edge between two vertices in the simplex.
We will denote by  $B_{r}^{L2}(u)$ the ball with radius $r$ (with respect to the L2 norm), centred at $u$. We will use $D^{L2}$ for the L2 unit sphere, and $D^{L1}$ for the L1 unit sphere. Lastly, let $V$ denote the version space, which is the set of hypotheses that are consistent with all the queries the algorithm used so far (note that $V\subseteq D^{L1}$ by our assumption that the hypotheses space is $D^{L1}$).

We assume the data is distributed uniformly over the unit sphere $D^{L2}$ in $\mathbb{R}^{n}$ (which is a common assumption for active learning of halfspaces \cite{10.1145/130385.130417, 10.5555/1577069.1577080,10.1007/978-3-540-72927-3_5}), and our hypothesis class is the set of homogeneous halfspaces. We represent every halfspace with its normal that resides on the L1 unit sphere $D^{L1}$. The reason why we choose to use the L1 unit sphere will be explained later on. We assume that there exists $w^{*}\in D^{L1}$, which separates the data perfectly (realizability). By that we mean that for every $x_{i}\in D^{L2}$ and let $y_{i}\in\{-1,1\}$ be its corresponding label, we have that  $sign(x_{i}\cdot w^{*})=y_{i}$.

\section{ALGORITHM}

The main idea of the algorithm is to find a simplex (hence the usage of the L1 unit sphere) that contains the current version space, and in each iteration bisect the longest edge between two vertices of that simplex, resulting in two smaller simplices, and continuing with one of them.
First of all, we find a simplex that contains the version space, by querying the standard unit vectors ($\bm{e_i}$). According to the labels of $\bm{e_i}$'s, we will find a simplex $S_1$ s.t $S_1 \subset D^{L1}$, whose vertices are $\pm{}\bm{e_i}$, and $V \subseteq S_1$.
After that, at every iteration the algorithm will query some $\bm{x_n} \in D^{L2}$ for its label $y_n$. We will choose $\bm{x_n}$ in such a way that only "half" of the hypotheses in $S_n$ will remain consistent with $y_n$. We achieve it by making sure $\bm{x_n}$ bisects the longest edge of the simplex, splitting it into two smaller simplices, $S_{n}^{+}$, $S_{n}^{-}$, where for every $\bm{w} \in S_{n}^{\pm}$ we have that $\bm{w} \cdot \bm{x_n} = \pm 1$ respectively. With that we can guarantee that $V \subseteq S_{n}^{y_n}$, and continue in a similar manner with $S_{n}^{y_n}$. See Figure 1.   

\begin{figure}
\begin{subfigure}{.6\textwidth}
  \centering
  \includegraphics[height=2.3cm,width=.6\linewidth]{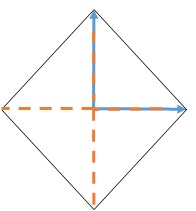}
  \caption{Start with the L1 unit sphere and query e1, e2}
  \label{fig:sfig1}
\end{subfigure}%
\begin{subfigure}{.6\textwidth}
  \centering
  \includegraphics[height=2.3cm,width=.6\linewidth]{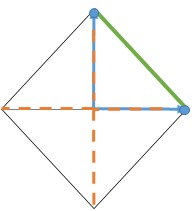}
  \caption{Assume both e1 and e2 were labeled as 1}
  \label{fig:sfig2}
\end{subfigure}
\begin{subfigure}{.6\textwidth}
  \centering
  \includegraphics[height=2.3cm,width=.6\linewidth]{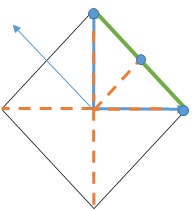}
  \caption{First iteration of the while loop}
  \label{fig:sfig3}
\end{subfigure}
\begin{subfigure}{.6\textwidth}
  \centering
  \includegraphics[height=2.3cm,width=.6\linewidth]{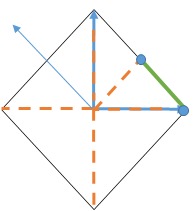}
  \caption{New simplex when $y_{new} = -1$}
  \label{fig:sfig4}
\end{subfigure}
\caption{Run example of the algorithm in $\mathbb{R}^{2}$. The diamond shape is the L1 unit sphere. Blue vectors are the queried data points. Blue dots are the current vertices, and green line is the current simplex that contains the version space. The orange dashed line represents the hyperplane that the current data point induces in the hypothesis space.}
\label{fig:fig}
\end{figure}

\begin{algorithm}[H]
\SetAlgoLined
\KwResult{Linear separator with error $\leq$  $\epsilon$}
\KwIn{$\epsilon$ generalization bound}
 Initialization: Initialize empty matrix $A_{n\times{n}}$ for the simplex vertices (each row of $A$ will correspond to a vertex in the simplex)\\
 
For every $1 \leq i \leq n$  query $\bm{e_i}$ (the standard unit base vectors) and get its label. if $y_{i} = 1$ set $\bm{e_i}$ to be the $i^{th}$ row of $A$, otherwise set it to -$\bm{e_i}$ \\
 
 \While{$d(S_n) > n\log_{\frac{\sqrt{3}}{2}}{\frac{\pi\epsilon}{2\sqrt{2n}}}$}{
  Find the longest edge between two vertices in $A$, and denote by $\bm{a_i}$ and $\bm{a_j}$ (the $i^{th}$ and $j^{th}$ rows of $A$, respectively) the vertices corresponding to that edge\\ 
  Calculate $\bm{v_{new}}$, the unit normal vector to the hyperplane spanned by $\bm{a_k}$ ($k \neq{} i,j$), and $\bm{a_{new}} = \frac{1}{2}\cdot(\bm{a_j} + \bm{a_i})$, and query its label $y_{new}$
  \eIf{ $y _{new} == sign( \bm{a_i} \cdot \bm{v_{new}})$ }{
  $\bm{a_j} = \bm{a_{new}}$
  }{
  $\bm{a_i} = \bm{a_{new}}$
  }
 }
 \caption{Version Space Minimizer}
 \Return $\frac{1}{n} \sum_{i=1}^{n}\bm{a_i}$ (the center of the simplex)
\end{algorithm}

\section{LABEL BOUND ANALYSIS}
In this section we will show a label bound for learning a linear separator, when the data is distributed uniformly over the unit sphere. As discussed earlier, this is a common scenario in the literature \cite{10.1145/130385.130417, 10.5555/1577069.1577080,10.1007/978-3-540-72927-3_5}.
In the appendix we show a simple bound in case the data is separable with margin $\gamma$ (and thus extending the analysis for other distributions).

\begin{theorem}[Label Bound For Version Space Minimizer]
\label{main}
Pick any $\epsilon > 0$. Assume the data is distributed uniformly over the unit sphere, and is perfectly labeled by some linear separator. Using the algorithm above we can get generalization error of $\epsilon$ after seeing only $O(n(\log{n} + \log{1/\epsilon}))$ labels.
\end{theorem}

Before proving the theorem, observe that since we assume the data is distributed uniformly over the unit sphere, the generalization error of every hypothesis $v\in D^{L1}$ is the angle between $v$ and $w^{*}$ divided by $\pi$ (as illustrated in Figure 2.):
\[  error(\bm{v}) = \frac{\arccos \frac{\bm{v}\cdot{}\bm{w^{*}}}{||\bm{v}||||\bm{w^{*}}||}}{\pi} = \frac{\theta}{\pi}\]

\begin{figure}[htp]
    \centering
    \includegraphics[width=4cm]{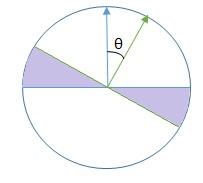}
    \caption{Disagreement region between two classifiers. $\theta$ is the angle between the two vectors.}
    \label{fig:error region}
\end{figure}

Intuitively, all points in $S_n$ are close to $w^{*}$ (when $S_n$ is small enough), and thus would have low generalization errors.

As a part of proving the main theorem, we will prove that the original simplex (at the end of step 2 of the algorithm), defined by the rows of $A$, contains the version space. In \ref{consistency_lemma}, we will also show that in every iteration of the while loop, the new simplex also contains the current version space.

\begin{lemma}
Let $S_1$ be the simplex we got at the end of step 2 of Version Space Minimizer (i.e the convex hull of the rows of A). For every $\bm{b}\in D^{L1}$ that was consistent on $\bm{e_i}$ for every $1 \leq i \leq n$ , we have that $\bm{b} \in S_1$.
\end{lemma}

\begin{proof}
By assumption we have that for every $1 \leq i \leq n$ , $sign(\bm{b} \cdot{} \bm{e_i}) = y_i$. This implies that $sign(\bm{b} \cdot{} \bm{e_i}) = sign(b_i) = y_i$. By definition of “Version Space Minimizer” we have that $a_i$ , the $i^{th}$ row of A, satisfies that $a_i = y_i \cdot{} \bm{e_i}$. Therefore, for every i, take $|b_i|a_i$, and we see that $|b_i|\bm{a_i} =|b_i|y_i\bm{e_i} =|b_i|sign(b_i)\bm{e_i} = b_i\bm{e_i}$ and since $\sum_{i=1}^{n} |b_i|=1$ we'll get that $\bm{b}= \sum_{i=1}^{n} b_i\bm{e_i} = \sum_{i=1}^{n} |b_i|\bm{a_i} \in S_1$ 
\end{proof}

\begin{lemma}
\label{consistency_lemma}
At every iteration $n \geq 0$ of the while loop, we have that every hypothesis $\bm{b} \in S_n$  that was consistent on $\bm{v_{new}}$ (i.e $sign(\bm{b} \cdot \bm{v_{new}}) = y_{new}$) satisfies that $\bm{b} \in S_{n+1}$ (the convex hull of the rows of A at the end of the iteration).
\end{lemma}

\begin{proof}
Write $\bm{b}$ as: 
\[\bm{b} = \sum_{k=1}^{n} c_k\bm{a_k}\] 
such that for every $1 \leq k \leq n$ we have that $c_k >0$, and $\sum_{k=1}^{n} c_k = 1$. We need to find $d_k$ such that:
\[\bm{b} = \sum_{k=1}^{n} d_k\bm{a'_k}\] where $\bm{a'_k}$ is the $k^{th}$ row of A at the end of the iteration, and $d_k$ satisfies that $d_k > 0 ,\sum_{k=1}^{n} d_k = 1$.
Observe that: 
\[\bm{b} \cdot \bm{v_{new}} = (\sum_{k=1}^{n} c_k\bm{a_k}) \cdot \bm{v_{new}}  = (c_i\bm{a_i}) \cdot \bm{v_{new}} +(c_j\bm{a_j}) \cdot \bm{v_{new}}\] where the second equality holds since for every $1 \leq k \leq n$ such that $k \neq i,j$, we have by the definition of $\bm{v_{new}}$ that $\bm{v_{new}} \cdot \bm{a_k} = 0$. We also have that: 
\[\bm{a_{new}} \cdot \bm{v_{new}} = \frac{1}{2}(\bm{a_j} + \bm{a_i}) \cdot \bm{v_{new}} = 0\] which implies that $\bm{a_i} \cdot \bm{v_{new}} = -(\bm{a_j} \cdot \bm{v_{new}})$. W.l.o.g assume that $y_{new} = 1$. We get that: 
\[1 = y_{new} = sign(\bm{b} \cdot \bm{v_{new}}) = sign((c_i\bm{a_i}) \cdot \bm{v_{new}} + (c_j\bm{a_j}) \cdot \bm{v_{new}})\] 
Therefore:
\begin{equation}
\begin{split}
(c_i\bm{a_i}) \cdot \bm{v_{new}} +(c_j\bm{a_j}) \cdot \bm{v_{new}} & =
(c_i\bm{a_i}) \cdot \bm{v_{new}} - (c_j\bm{a_i}) \cdot \bm{v_{new}} \\
& = (\bm{a_i} \cdot \bm{v_{new}})(c_i - c_j) \geq 0 
\end{split}
\end{equation}
Let us consider two cases. If $\bm{a_i} \cdot \bm{v_{new}} \geq 0$ in order for the above inequality to hold we must have that $c_i \geq c_j$. Furthermore, since $sign(\bm{a_i} \cdot \bm{v_{new}}) = 1 = y_{new}$, we get by the definition of the algorithm that the $j^{th}$ row of $A$ will be set to $\bm{a_{new}}$. Therefore, for every $1 \leq k \leq n$ such that $k \neq i,j$, let $d_k = c_k$ and $d_i = c_i - c_j$, $d_j =2c_j$.
Note that $d_t \geq 0$, for every $1 \leq t \leq n$ (since $c_i \geq c_j$), and also that: $\sum_{k=1}^{n} d_k = \sum_{k=1}^{n} c_k = 1$. 
Since for every $1 \leq k \leq n$ such that $k \neq i,j$, we have that $a_k =a'_k$ as well as $d_k = c_k$, it's enough to see that:
\begin{equation}
\begin{split}
d_i\bm{a'_i} + d_j\bm{a'_j} & = (c_i - c_j)\bm{a_i} +2c_j\bm{a_{new}} \\
& = (c_i - c_j)\bm{a_i} + 2c_j(\frac{1}{2}(\bm{a_j} + \bm{a_i})) = c_i\bm{a_i} + c_j\bm{a_j}
\end{split}
\end{equation}

Therefore, we get that $b = \sum_{k=1}^{n} c_k\bm{a_k} = \sum_{k=1}^{n} d_k\bm{a'_k}$. The proof in the case that $\bm{a_i} \cdot \bm{v_{new}} < 0$ is similar (in that case we have $c_j \geq c_i$, and $\bm{a'_i} = \bm{a_{new}}$).
\end{proof}

\begin{corollary}
\label{cor_1}
We know by the realizability assumption that $w^{*} \in V$, and thus is also contained in the returned simplex.
\end{corollary}

We will now prove a Lemma for bounding the maximum angle between two vectors in the returned simplex, as a function of the diameter and the dimension. (note that this will bound the generalization error as we know that $\bm{w^{*}}$ is in that simplex).

\begin{lemma}
\label{max_angle_lemma}
For small enough $\delta > 0$, assume $d(S_n) \leq \delta$. For every $\bm{w},\bm{u} \in S_n$ the angle $\theta$ between them is bounded by $2\delta\sqrt{n}$.
\end{lemma}

\begin{proof}
Let $\bm{w},\bm{u} \in S_n$. Note that since $d(S_n) \leq \delta$ we get that $S_n \subseteq B_{\delta}^{L2}(\bm{u})$.
Observe that in general, for every $\bm{v} \in \mathbb{R}^{n}$ we have that the maximum angle between $\bm{v}$ and every other $\bm{r} \in B_{\delta}^{L2}(\bm{v})$ is $\arcsin{\frac{\delta}{||v||}}$ (Figure 3.). Since $\bm{w} \in S_n \subseteq B_{\delta}^{L2}(\bm{u})$  we get that:
\begin{equation}
\label{num_3}
\theta \leq \arcsin{\frac{\delta}{||u||}}   
\end{equation}

Since $0 \leq \arcsin{\frac{\delta}{||u||}} \leq \frac{\pi}{2}$ for small enough $\delta$, and $\sin$ is monotonic increasing on $[0,\frac{\pi}{2}]$ we can apply $\sin$ to \ref{num_3} and get:
\begin{equation}
\sin{\theta} \leq \frac{\delta}{||u||}   
\end{equation}
Furthermore, since $\bm{u} \in D^{L1}$ we get that $||\bm{u}|| \geq \frac{1}{\sqrt{n}}$ (Since the minimum norm on the L1 unit sphere is obtained when all the coordinates = $\frac{1}{n}$) and so we get that:
\begin{equation}
\sin{\theta} \leq \delta \sqrt{n}  
\end{equation}

Finally, for small enough $\delta$ we have that $2\sin{\theta} > \theta$ and so:
\begin{equation}
\theta \leq 2\sin{\theta} \leq 2\delta \sqrt{n}  
\end{equation}

\end{proof}

\begin{figure}[htp]
    \centering
    \includegraphics[width=6cm]{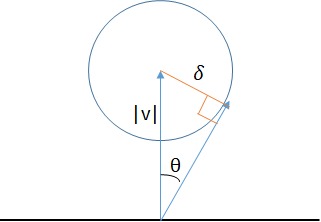}
    \caption{Maximum angle within $\delta$ radius ball }
    \label{fig:maximum angle}
\end{figure}

We will make use of the following theorem that was proven by Baker Kearfott in \cite{10.2307/2006341} regarding the decrease rate of the diameter of the simplex obtained by the algorithm.

\begin{theorem}
\label{diam_bound}
Let $S$ be an n-simplex in $\mathbb{R}^{n}$ and $S_p$ is any simplex produced after p bisections of $S$ (longest edge bisections), then the diameter of $S_p$ is no greater than $(\frac{\sqrt{3}}{2})^{\floor{\frac{p}{n}}}$ times the diameter of $S$.
\end{theorem}

We are now ready to prove theorem \ref{main}.
\begin{proof}
By corollary \ref{cor_1} we know that $w^{*}$ is contained in the returned simplex. Together with lemma \ref{max_angle_lemma} we know that if the returned simplex $S_n$ satisfies that $d(S_n) \leq \delta$ we have that the maximum angle between $w^{*}$ and any other $u \in S_n$ is bound by:
\begin{equation}
\theta \leq 2\delta \sqrt{n}  
\end{equation}
Since for every $u \in S_n$ we have that:
\[error(u) = \frac{\arccos \frac{u\cdot{}w^{*}}{||u||||w^{*}||}}{\pi} = \frac{\theta}{\pi}\]
So to ensure at most $\epsilon$ error we can bound:
\begin{equation}
\frac{2\delta \sqrt{n}}{\pi} < \epsilon  
\end{equation}
Rearranging, we get:
\begin{equation}
\delta < \frac{\pi\epsilon{}}{2\sqrt{n}}  
\end{equation}

Notice that the diameter of $S_1$ (the initial simplex we get after step 2 of the algorithm) is equal to $\sqrt{2}$ (since the distance between $\pm \bm{e_i}, \bm{e_j}$, such that $i\neq j$ is $\sqrt{2}$).
Therefore, in order to bound the diameter we will apply theorem \ref{diam_bound} and find $p$ such the following holds:
\begin{equation}
\sqrt{2}(\frac{\sqrt{3}}{2})^{\floor{\frac{p}{n}}} < \frac{\pi\epsilon{}}{2\sqrt{n}} 
\end{equation}
Rearrange:
\begin{equation}
\label{fianl_task}
 p > n\log_{\frac{\sqrt{3}}{2}}{\frac{\pi\epsilon}{2\sqrt{2n}}} = O(n(\log{n} + \log{1/\epsilon}))   
\end{equation}
And we obtained that after $p$ iterations for a $p$ which \ref{fianl_task} holds for, the generalization error of every hypothesis in $S_p$, is at most $\epsilon$.
Finally note that the amount of labels the algorithm queries is equal to:
\[p + n\]
Where $p$ is the number of iterations of the while loop (and note that we query one sample in each iteration), and $n$ queries for step 2. Therefore, the total number of labels used by the algorithm is also:
\[O(n(\log{n} + \log{1/\epsilon}))\]

\end{proof}

It is important to note that this bound is not probabilistic, unlike previous results involving active learning of halfspaces.
\begin{remark}
Note that it can be shown that $\Omega(n\cdot(\log{\frac{1}{\epsilon}}))$ is a lower bound for learning halfspaces (using a sphere counting argument, see \cite{10.5555/1577069.1577080}), therefore the algorithm's label usage is near optimal.
\end{remark}

\section{EXPERIMENTS}

In this section we will show the effectiveness of our proposed algorithm on synthetic data (distributed uniformly over $D^{L2}$). In particular, we will compare our algorithm with uncertainty sampling, as well as random sampling from $D^{L2}$.

\begin{remark}
In order to calculate the normal vector in every iteration, we used the SVD decomposition of A (the vertices of the simplex), which is the most demanding (computational wise) part of the algorithm. Therefore, the running time of the algorithm grows polynomially with the dimension.
\end{remark}

In figure \ref{fig:experiment}, we compare the results of our algorithm with the classic idea of uncertainty sampling. In our case, since we allow membership queries, in every iteration we query a random vector that is orthogonal to the current classifier, and then compute SVM on the extended training set (this is equivalent for querying a vector with $0$ distance from the separating hyperplane). In figure \ref{fig:experiment} it can be seen that our method significantly out-performs uncertainty sampling, and achieving exponential gain when compared to random sampling.

\begin{figure}[H]
     \centering
     \begin{subfigure}[b]{0.45\textwidth}
         \centering
         \includegraphics[width=\textwidth]{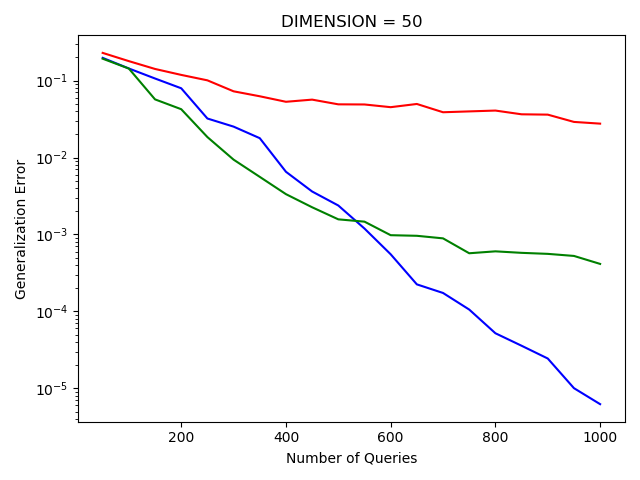}
         \label{fig:y equals x}
     \end{subfigure}
     \hfill
     \begin{subfigure}[b]{0.45\textwidth}
         \centering
         \includegraphics[width=\textwidth]{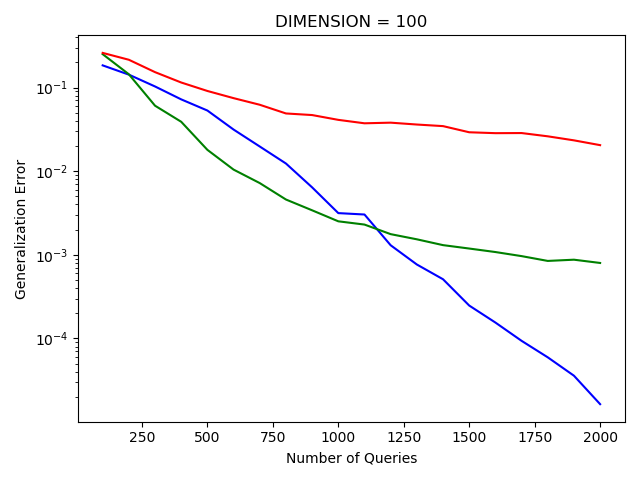}
         \label{fig:three sin x}
     \end{subfigure}
     \caption{In this graph the estimation error (in log scale) is plotted against the number of queries. Red: Random Sampling, Green: Uncertainty Sampling, Blue: Our Algorithm.}
     \label{fig:experiment}
\end{figure}


\section{CONCLUSION}
We presented the "Version Space Minimizer" algorithm for learning homogeneous halfspaces, and analyzed it's label complexity - showing it provably achieves near optimal label complexity when the input space is uniformly distributed over the unit hypersphere. We also demonstrated that the algorithm works well in practice, and significantly out-performs uncertainty sampling.

\bibliographystyle{unsrt}
\bibliography{bib}

\section{Appendix}
\subsection*{Label bound for separability with margin}
In this appendix we will show a simple extension for the label bound analysis, for distributions separable with margin.
\begin{definition*}
We say that distribution $D$ is separable with margin $\gamma$, if there exists $\bm{w} \in \mathbb{R}^n$ such that $||\bm{w}||_2 = 1$ and such that with probability 1 over the choice of $\bm{x} \sim D$, we have that $y(\bm{w}\cdot\bm{x}) \geq \gamma$.
\end{definition*} 
We will prove the following Lemma:
\begin{lemma}
Let $D$ be a distribution separable with margin $\gamma$. Using the algorithm above we can get generalization error of $0$ after seeing only $O(n(\log n + \log R + \log \frac{1}{\gamma}))$ labels, where $n$ is the dimension, and $R$ is a constant such that for every $\bm{x} \sim D$ we have $||\bm{x}||_2 \leq R$.
\end{lemma}

Since for every $\bm{w} \in \mathbb{R}^n$ we have that: $||\bm{w}||_1 \leq \sqrt{n}\cdot||\bm{w}||_2$, it implies that for a distribution $D$, which is separable by $\bm{h}$ with $||\bm{h}||_2 = 1$ we have:
$$y(\bm{h}\cdot\bm{x}) \geq \gamma \longrightarrow y(\frac{\bm{h}}{||\bm{h}||_1}\cdot\bm{x}) \geq \frac{\gamma}{||\bm{h}||_1} \geq \frac{\gamma}{\sqrt{n}} $$

For brevity, we will denote $\bm{h^{*}} = \frac{\bm{h}}{||\bm{h}||_1}$.
This time, we will continue running the algorithm until $d(S_n) \leq \frac{\gamma}{2\sqrt{n}R}$, where $||\bm{x}||_2 \leq R$ for every $\bm{x} \sim D$. By \ref{diam_bound}, we can guarantee it happens after p iterations, for every p such that:
$$ p \geq n\log_{\frac{\sqrt{3}}{2}}{\frac{\gamma}{2\sqrt{2n}R}}$$
In this case, when the algorithm returns $\bm{w}$ after using $p + n$ labels, since 
$d(S_n) \leq \frac{\gamma}{2\sqrt{n}R}$ we are guaranteed that $||\bm{h^*} - \bm{w}||_2 \leq \frac{\gamma}{2\sqrt{n}R}$. Therefore, for every $\bm{x} \sim D$ with label $y$ we have:
\begin{equation}
\begin{split}
    y(\bm{w}\cdot{\bm{x}}) & = y((\bm{w} + \bm{h^{*}} - \bm{h^{*}}) \cdot{\bm{x}}) = y(\bm{h^{*}}\cdot\bm{x}) + y((\bm{w} - \bm{h^{*}})\cdot\bm{x}) \\ 
    & \geq \frac{\gamma}{\sqrt{n}} + y((\bm{w} - \bm{h^{*}})\cdot\bm{x}) \geq  \frac{\gamma}{\sqrt{n}} - ||\bm{w}-\bm{h^{*}}||_2||\bm{x}||_2 \\
    & \geq \frac{\gamma}{\sqrt{n}} - \frac{\gamma}{2\sqrt{n}R} \cdot R = \frac{\gamma}{2\sqrt{n}}
\end{split}
\end{equation}

Where the first inequality comes from separability with margin, and the second one comes from Cauchy-Schwarz inequality.
So in general we see that for every $\bm{x} \sim D$ it holds that:
$$y(\bm{w}\cdot{\bm{x}}) \geq \frac{\gamma}{2\sqrt{n}} > 0$$
Which implies that the generalization error would be 0 in this case.
Since we have:
$$p + n =  n\log_{\frac{\sqrt{3}}{2}}{\frac{\gamma}{2\sqrt{2n}R}} + n = O(n(\log n + \log R + \log \frac{1}{\gamma}))$$
So we showed that $O(n(\log n + \log R + \log \frac{1}{\gamma}))$ labels are sufficient in order to reach generalization error of 0, when $D$ is separable with margin $\gamma$, and for every $\bm{x} \sim D$ we have $||\bm{x}||_2 \leq R$.



\end{document}